\pdfoutput=1 
\documentclass[letterpaper]{article} 
\usepackage[preprint]{aaai2027}  
\usepackage[hyphens]{url}  
\usepackage{graphicx} 
\usepackage{natbib}  
\usepackage{caption} 
\usepackage{algorithm}
\usepackage{algorithmic}

\usepackage{macro}
\usepackage{outlines}
\usepackage{todonotes}
\usepackage{amssymb}
\usepackage{amsmath}
\usepackage{amsthm}
\usepackage{tikz}
\usetikzlibrary{positioning, arrows.meta}
\usepackage{booktabs}

\theoremstyle{plain}
\newtheorem{theorem}{Theorem}

\newtheorem{lemma}[theorem]{Lemma}
\newtheorem{corollary}[theorem]{Corollary}

\theoremstyle{definition}
\newtheorem{definition}{Definition}

\theoremstyle{remark}

\usepackage{newfloat}
\usepackage{listings}
\DeclareCaptionStyle{ruled}{labelfont=normalfont,labelsep=colon,strut=off} 
\floatstyle{ruled}
\newfloat{listing}{tb}{lst}{}
\floatname{listing}{Listing}

\title{Learning Lookahead Lemmas for Neural Network Verification}
\author{
    Liam Davis,
    Haoze Wu
}
\affiliations{
    Amherst College
}

\usepackage{docmute}

\begin{document}

\maketitle

\begin{abstract}
State-of-the-art neural network verifiers use the branch-and-bound procedure
as their core solving mechanism. We introduce an inprocessing
framework for neural network verification driven by the lookahead procedure.
Under this framework, lookahead derives new lemmas over the phases of unstable
ReLUs, which are collected into an implication graph that is used to
prune the search space and vivify boolean cuts. We instantiate the framework
in two state-of-the-art verifiers, Marabou and \abcrown, and demonstrate that it
improves performance in both, proving up to 34\% more instances unsatisfiable.
\end{abstract}

\section{Introduction}
\label{sec:intro}

Deep neural networks (DNNs) have become state-of-the-art solutions in various
domains~\citep{sallab2017deep,mnih2013playing,he2015delving}.
Despite this, their internal structure is opaque or unintelligible, which makes it difficult 
to understand their behavior and trust their predictions. To ensure the safe deployment 
of DNNs in safety-critical domains, there has been substantial effort from the formal 
methods and machine learning communities on developing scalable formal verification 
techniques~\citep{katz2017reluplex,singh2019deeppoly,zhang2018efficient,wang2021beta}.
State-of-the-art complete neural network verifiers are based on the
branch-and-bound (BaB) procedure, which involves performing case splitting on 
non-linear activation functions, such as ReLUs, and analyzing the resulting 
subproblems with an incomplete verifier. If the analysis is inconclusive, 
further case splits are performed recursively.

Recent efforts to accelerate the BaB search procedure fall broadly into two families.
The first family focuses on improving the branching heuristic, or the strategy
used to select what ReLU to perform a case split on. Heuristics such as BaBSR~\citep{bunel2020branch}
and pseudo-impact~\citep{wu2022efficient} attempt to get the most out of the information already available
about the state of the search. Lookahead
branching~\citep{davis2026lookaheadbranchingneuralnetwork} extends this idea into a general
procedure that simulates several candidate splits and evaluates the subproblems each one
produces. The second family is deriving information from subproblems proven to be infeasible,
usually in the form of boolean cuts like BICCOS~\citep{zhou2024scalable} and PICID~\citep{isac2026picidproofdrivenclauselearning}.
Both of these families are reactive, in that they exploit information the search has already
paid to obtain. Lookahead branching~\citep{davis2026lookaheadbranchingneuralnetwork} comes closest to escaping this, as simulating
a candidate split can fix the phase of an unstable ReLU (active or inactive) globally,
but this is a consequence of the procedure rather than its primary purpose, which remains
to select the best ReLU to split on. Compare this to modern SAT solvers, which not only derive
information reactively via conflict clauses, but also proactively through
inprocessing~\citep{jarvisalo2012inprocessing}, where
the search is periodically paused to conduct exploration intended to derive new
information. Neural network verification has no analogue, and this is the gap we seek to fill.

In this work, we explore the idea of inprocessing in neural network verification in full,
treating the derivation of new information as a part of the solve procedure
rather than as a consequence of the branching heuristic.
Specifically, we construct a variation of the lookahead procedure to
derive lemmas about the phases of unstable ReLUs. Each lemma states either that a phase
is infeasible outright or that fixing one phase forces another, and we collect them into an
\emph{implication graph}. The graph is not a fixed set of facts but a structure the solver reasons
over, and we develop three techniques that do so. \emph{SAT closure} treats the graph as a
propositional formula and closes it under its logical consequences, so that a subproblem can be
pruned whenever the phase fixes are jointly unsatisfiable. \emph{Reprobing} reruns the lookahead procedure as the search
establishes new phases, refreshing the graph with implications that hold only once bounds have
tightened. \emph{Cut vivification} uses the graph to extract an unsatisfiable core of a boolean
cut, shortening the cut to the phase fixes responsible for the infeasibility so that it
prunes more subproblems. Together, these compound into an inprocessing framework for neural
network verification. We implement the framework in two state-of-the-art verifiers,
Marabou~\citep{katz2019marabou,wu2024marabou2} and
\abcrown~\citep{zhang2018efficient,xu2020automatic,xu2021fast,wang2021beta,zhang2022general,zhang22babattack,kotha2023provably,shi2025genbab,zhou2025clip},
and show that the framework yields improved performance in both verifiers.

The rest of the paper is organized as follows: Section~\ref{sec:related} reviews related work. 
Section~\ref{sec:background} provides background information on neural network verification. 
Section~\ref{sec:methodology} presents the construction of the implication graph with the lookahead
procedure, as well as how the graph is used to prune the search space and strengthen boolean cuts.
Section~\ref{sec:evaluation} provides an evaluation of the inprocessing framework, comparing
it to the current version of both Marabou and \abcrown. Finally, we conclude and discuss 
future directions in Section~\ref{sec:conclusion}.

\section{Related Work}
\label{sec:related}

Early efforts for complete neural network verification used SMT and MILP solvers to 
enumerate activation patterns~\citep{katz2017reluplex}. The first unified instantiation 
of BaB was presented by~\citet{bunel2020branch}. State-of-the-art complete neural 
network verifiers include SMT-based solvers~\citep{katz2019marabou, wu2024marabou2} and 
GPU-accelerated bound propagation-based verifiers~\citep{wang2021beta,shi2025genbab}.

Inprocessing techniques have been widely used in both SAT and SMT solving, including
variable and clause elimination~\citep{een2005effective}, probing-based simplifications such as
failed literal detection, hyper-binary resolution and equality
reduction~\citep{bacchus2003effective}, and simplifications derived from the binary implication
graph~\citep{heule2011efficient}. ~\citet{biere2021preprocessing} provides a survey of such techniques. While these
techniques were first applied as preprocessing, \citet{jarvisalo2012inprocessing} formalized the
rules under which they can be interleaved with the search itself while preserving satisfiability.
Clause vivification was likewise introduced as a preprocessing
step~\citep{piette2008vivifying} and later applied to learnt clauses during the
search~\citep{luo2017effective,li2020clause}. Lookahead has also been used to derive information
rather than only to branch, most notably in cube-and-conquer~\citep{heule2011cube}.

Our contributions draw conceptually from this line of work. The implication graph is inspired by
the binary implication graph, reprobing by periodic re-simplification, and cut vivification by
clause vivification. We depart from it in how the information is obtained, taking inspiration
from lookahead branching~\citep{davis2026lookaheadbranchingneuralnetwork} and using a variant of
the lookahead procedure as the probe, which lets the framework reason about ReLU phases with the
bound propagation machinery a verifier already has.

\section{Background}
\label{sec:background}

\subsection{Neural Network Verification}
\label{subsec:nnv}
For a trained deep neural network $N: \mathbb{R}^n \rightarrow \mathbb{R}^m$ with an 
input $x \in \mathbb{R}^n$ and output $y \in \mathbb{R}^m$, the general DNN verification problem 
is whether or not there exists an input $x$ that produces an output 
$y$ that satisfies a property $\phi(y)$. If there exists an input $x$ that leads to an 
output $y$ that satisfies the property $\phi(y)$, then the problem is satisfiable 
(SAT); otherwise it is unsatisfiable (UNSAT). As the property $\phi(y)$ 
is a safety property, SAT denotes that an adversarial example violates the safety property 
and is referred to as \emph{unsafe}, while UNSAT denotes a satisfied property and is referred to as \emph{safe}.
In this paper, we will use the terms SAT and UNSAT.

\subsection{Branch-and-Bound}
\label{subsec:bab}
The branch-and-bound (BaB) framework, illustrated by Figure~\ref{fig:bab}, is an efficient approach to neural network
verification. BaB systematically tightens the bounds of a neural network by splitting on unstable ReLU neurons. When an unstable
ReLU is split, the problem is divided into two subproblems, one where the ReLU is active and one where the ReLU is
inactive. After bound propagations, this turns a large problem into two more manageable problems. To verify with BaB,
ReLU splits are repeatedly applied until each resulting subproblem can be classified as either SAT or UNSAT.

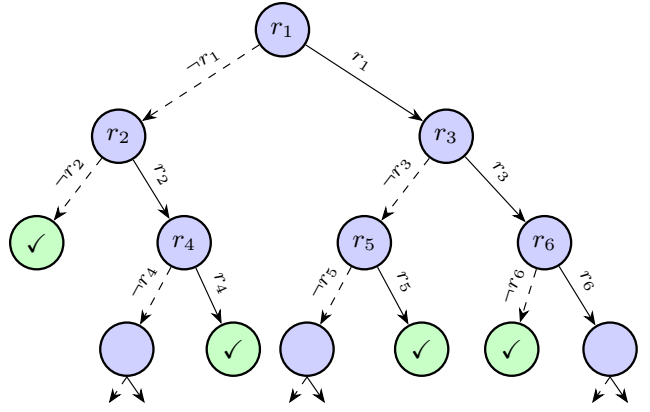
\begin{figure}[t]
    \centering
    \resizebox{\linewidth}{!}{\begin{tikzpicture}[
    subproblem/.style={circle, draw, thick, minimum size=6.5mm, inner sep=0pt, font=\small},
    unverified/.style={subproblem, fill=blue!18},
    verified/.style={subproblem, fill=green!22},
    edgelabel/.style={font=\scriptsize, sloped},
    ->/.style={-{Stealth[length=1.8mm]}}
]

\node[unverified] (nRoot) at (0,0)      {$r_1$};

\node[unverified] (nL) at (-2.0,-1.3)   {$r_2$};
\node[unverified] (nR) at ( 2.0,-1.3)   {$r_3$};

\node[verified]   (nA) at (-3.0,-2.6)   {\checkmark};
\node[unverified] (nB) at (-1.2,-2.6)   {$r_4$};
\node[unverified] (nC) at ( 1.0,-2.6)   {$r_5$};
\node[unverified] (nD) at ( 3.2,-2.6)   {$r_6$};

\node[unverified] (nE) at (-1.9,-3.9)   {};
\node[verified]   (nF) at (-0.6,-3.9)   {\checkmark};
\node[unverified] (nG) at ( 0.3,-3.9)   {};
\node[verified]   (nH) at ( 1.7,-3.9)   {\checkmark};
\node[verified]   (nI) at ( 2.8,-3.9)   {\checkmark};
\node[unverified] (nJ) at ( 4.0,-3.9)   {};

\draw[->, dashed] (nRoot) -- (nL) node[edgelabel, pos=0.4, above] {$\neg r_1$};
\draw[->]         (nRoot) -- (nR) node[edgelabel, pos=0.4, above] {$r_1$};

\draw[->, dashed] (nL) -- (nA) node[edgelabel, pos=0.4, above] {$\neg r_2$};
\draw[->]         (nL) -- (nB) node[edgelabel, pos=0.4, above] {$r_2$};

\draw[->, dashed] (nR) -- (nC) node[edgelabel, pos=0.4, above] {$\neg r_3$};
\draw[->]         (nR) -- (nD) node[edgelabel, pos=0.4, above] {$r_3$};

\draw[->, dashed] (nB) -- (nE) node[edgelabel, pos=0.4, above] {$\neg r_4$};
\draw[->]         (nB) -- (nF) node[edgelabel, pos=0.4, above] {$r_4$};

\draw[->, dashed] (nC) -- (nG) node[edgelabel, pos=0.4, above] {$\neg r_5$};
\draw[->]         (nC) -- (nH) node[edgelabel, pos=0.4, above] {$r_5$};

\draw[->, dashed] (nD) -- (nI) node[edgelabel, pos=0.4, above] {$\neg r_6$};
\draw[->]         (nD) -- (nJ) node[edgelabel, pos=0.4, above] {$r_6$};

\foreach \n in {nE, nG, nJ} {
    \draw[->, dashed] (\n.south) -- ++(-0.22,-0.32);
    \draw[->]         (\n.south) -- ++( 0.22,-0.32);
}

\end{tikzpicture}}
    \caption{Each node represents a subproblem from BaB by splitting unstable ReLUs. 
    Green nodes are verified, blue nodes require further splitting.}
    \label{fig:bab}
\end{figure}

\subsection{Boolean Cuts in Branch-and-Bound}
\label{subsec:cuts}
In the BaB framework, a \emph{cut} is a global infeasibility that can be used to determine that a subproblem is UNSAT, 
without requiring further ReLU splits. Boolean cuts are collections of ReLU phase fixes that are globally infeasible.

In neural network verification, there are two predominant boolean cuts: BICCOS~\citep{zhou2024scalable} and
PICID~\citep{isac2026picidproofdrivenclauselearning}. \abcrown uses BICCOS cuts, which are infeasibilities mined
from the BaB search procedure. Formally, let $R_+$ and $R_-$ denote the sets of
neurons restricted to the active ($r_i = 1$) and inactive ($r_i = 0$) regimes, respectively,
along a branch proven UNSAT. BICCOS introduces the cut
\begin{equation}
    \label{eq:biccos}
    \sum_{i \in R_+} r_i - \sum_{i \in R_-} r_i \le |R_+| - 1,
\end{equation}
where $r_i \in [0,1]$ is the relaxed ReLU indicator for neuron $i$. This inequality is violated exactly by the
assignment $r_i = 1\ \forall i \in R_+$ and $r_i = 0\ \forall i \in R_-$, so it excludes that phase
combination everywhere in the search tree, even in subdomains where those neurons have not yet been split.

PICID is used in Marabou, and derives a similar boolean cut through a different mechanism. Whenever a subproblem is proven UNSAT, PICID uses the
Farkas lemma to derive an infeasibility certificate that isolates the branching decisions responsible
for the contradiction. Formally, let $R$ denote a set of ReLU phase fixes that
jointly derive infeasibility. PICID defines the cut as the conjunction over $R$:
\begin{equation}
    \bigwedge_{i \in R} r_i.
\end{equation}

Both cuts are applied the same way during the BaB search procedure. When a subproblem is proven
UNSAT, the cut is added to the global problem and if violated elsewhere, that subproblem is immediately deemed UNSAT without further ReLU splits.
This way, cuts can be used to prune the search tree and accelerate the verification process.

\subsection{Lookahead Branching}
\label{subsec:lookahead-branching}
Lookahead Branching~\citep{davis2026lookaheadbranchingneuralnetwork} is a branching heuristic for BaB that selects which unstable ReLU to split at a given subproblem. The heuristic
simulates several splits prior to committing to one, making more informed decisions than
heuristics that rely solely on local information. Notably, it can also fix the phases of
certain unstable ReLUs.

In this work, we extend this direction and instead of using lookahead as a branching heuristic, we use it 
solely to derive new information, integrating it with the boolean cuts in Section~\ref{subsec:cuts} into 
our framework. 

\section{Methodology}
\label{sec:methodology}

Recall that a phase fix pins an unstable ReLU to its \emph{active} or \emph{inactive}
state. Denote $r_i$ for a phase fix of ReLU $i$ in the active phase and $\neg r_i$ for the complementary
phase fix in the inactive phase.

\begin{definition}[Binary Implication]
\label{def:binary-implication}
Let $r_1$ and $r_2$ be phase fixes of unstable ReLUs. We say $r_1$ \emph{implies}
$r_2$, written $r_1 \to r_2$, if the phase fixes $\{r_1, \neg r_2\}$ are globally infeasible.
\end{definition}

Binary implications are the natural unit of information here, as pinning one neuron shows what that
one fix forces, whereas obtaining larger implications would require probing sets of fixes.

\begin{definition}[Unit Lemma]
\label{def:unit-lemma}
Let $r$ be a phase fix of an unstable ReLU. We say $r$ is a \emph{unit lemma} if the phase fix
$\{\neg r\}$ is globally infeasible on its own, i.e., $r$ is entailed unconditionally.
\end{definition}

\begin{definition}[Implication Graph]
\label{def:implication-graph}
The \emph{implication graph} is the directed graph $G = (V, E)$, where
$V = \{r_i, \neg r_i : i \text{ is an unstable ReLU}\}$ is the set of phase fixes and
$E = \{(r, r') \in V \times V : r \to r'\}$ is the set of binary implications.
\end{definition}

A unit lemma is represented in the implication graph as an edge from a phase fix to its
complement. Applying Definition~\ref{def:binary-implication} to the two phase fixes of the same
ReLU reduces exactly to $\{\neg r\}$ being globally infeasible, so a unit lemma $r$ contributes
the edge $\neg r \to r$ to $E$.

\begin{figure}[t]
    \centering
    \resizebox{\linewidth}{!}{\begin{tikzpicture}[
    literal/.style={circle, draw, thick, minimum size=7.5mm, inner sep=0pt, font=\small, fill=blue!18},
    unit/.style={literal, fill=orange!22},
    edgelabel/.style={font=\scriptsize, sloped},
    ->/.style={-{Stealth[length=1.8mm]}}
]

\node[literal] (r1)  at (0,0)     {$r_1$};
\node[literal] (nr1) at (0,-1.8)  {$\neg r_1$};

\node[literal] (r2)  at (2.6,0)     {$r_2$};
\node[literal] (nr2) at (2.6,-1.8)  {$\neg r_2$};

\node[literal] (r3)  at (5.2,0)     {$r_3$};
\node[literal] (nr3) at (5.2,-1.8)  {$\neg r_3$};

\node[unit]    (r4)  at (7.8,0)     {$r_4$};
\node[literal] (nr4) at (7.8,-1.8)  {$\neg r_4$};

\draw[->]         (r1)  -- (r2)  node[edgelabel, pos=0.5, above] {};
\draw[->, dashed]  (nr2) -- (nr1) node[edgelabel, pos=0.5, below] {};

\draw[->]         (r1)  -- (nr3) node[edgelabel, pos=0.6, above] {};
\draw[->, dashed]  (r3)  -- (nr1) node[edgelabel, pos=0.4, below] {};

\draw[->]         (nr2) -- (r4)  node[edgelabel, pos=0.6, below] {};
\draw[->, dashed]  (nr4) -- (r2)  node[edgelabel, pos=0.4, above] {};

\draw[->, very thick] (nr4) to[bend right=35] (r4);

\end{tikzpicture}}
    \caption{Implication graph over four unstable
    ReLUs. Solid edges are binary implications; dashed edges are their contrapositives.
    $r_4$ is a unit lemma.}
    \label{fig:implication-graph}
\end{figure}
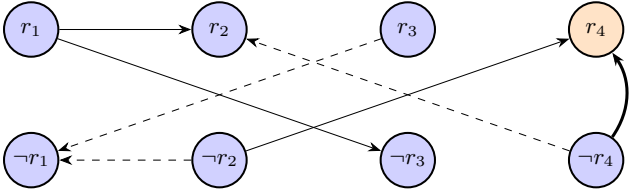

Figure~\ref{fig:implication-graph} illustrates an example implication graph over four unstable ReLUs.
An implication graph can be used to derive new information during the BaB search procedure, but
the implication graph must be constructed first. The next sections describe
how to construct the implication graph, then how to use it to derive new information.

\subsection{Constructing the Implication Graph}
\label{subsec:constructing-implication-graph}

Algorithm~\ref{alg:construct-implication-graph} shows the lookahead
procedure used to construct the implication graph. Each unstable ReLU is fixed to both
the active and inactive phases. After each phase fix, the bounds are tightened, but
not with the full verification machinery. Unlike lookahead branching, which commits 
to each simulated split using the full bound propagation procedure, constructing
the implication graph only requires a restricted, budgeted tightening pass per phase fix.
A restricted pass keeps each probe cheap enough to afford a probe over every unstable neuron.
Marabou tightens with DeepPoly~\citep{singh2019deeppoly}, plus a pivot-capped simplex feasibility check rather than an
exact LP. \abcrown tightens with a single $\alpha$-CROWN~\citep{xu2021fast} pass reusing build-time alphas
rather than the full optimization loop. This restriction is what makes
lookahead as inprocessing cheap in comparison to the branching heuristic.

After tightening, the bounds are checked for infeasibility. If the phase fix is infeasible, 
then we can record the complementary phase fix as a unit lemma. Otherwise, the previously
unstable neurons are checked to see if any have been forced to a single phase. 
Each newly-fixed neuron is assigned a new binary implication, where the 
source is the phase fix that forced it and the target is the newly-fixed neuron.
Theorem~\ref{thm:soundness-construction} formalizes the soundness of the implication graph.

\begin{algorithm}[t]
\caption{Lookahead Procedure for Constructing the Implication Graph}
\label{alg:construct-implication-graph}
\begin{algorithmic}[1]
\REQUIRE unstable ReLUs $1, \dots, n$
\ENSURE implication graph $G = (V, E)$
\STATE $V \gets \{r_i, \neg r_i : i = 1, \dots, n\}$; $E \gets \emptyset$
\FOR{each unstable ReLU $i$ and phase fix $r \in \{r_i, \neg r_i\}$}
    \STATE tighten bounds under $r$ with a restricted oracle
    \IF{$r$ is infeasible}
        \STATE record $\neg r$ as a unit lemma
    \ELSE
        \FOR{each unstable ReLU $j \neq i$}
            \IF{the tightened bounds force $j$ to a single phase $r_j'$}
                \STATE $E \gets E \cup \{(r, r_j')\}$
            \ENDIF
        \ENDFOR
    \ENDIF
\ENDFOR
\RETURN $G = (V, E)$
\end{algorithmic}
\end{algorithm}

\begin{theorem}[Soundness of Construction]
\label{thm:soundness-construction}
If the tightening oracle (CROWN / DeepPoly) is sound, then every unit lemma and every edge of the
resulting implication graph $G$ satisfies Definition~\ref{def:unit-lemma} and
Definition~\ref{def:binary-implication}, respectively.
\end{theorem}

Proof is deferred to Appendix~\ref{sec:appendix-proofs}. Through the lookahead procedure,
we can also derive tighter bounds on all ReLUs. When we look ahead on both phases of a neuron,
each phase fix outputs its own per-neuron bounds. Since every point satisfies one of the two
phases, the hull of the two bounds is sound and can be tighter than the bounds at the root. The
following theorem formalizes this bound.

\begin{theorem}[Hull Bound]
\label{thm:hull-bound}
For each unstable ReLU $i$ and neuron $k$, let $[\mathrm{lb}_k^{r_i}, \mathrm{ub}_k^{r_i}]$ and
$[\mathrm{lb}_k^{\neg r_i}, \mathrm{ub}_k^{\neg r_i}]$ be the bounds computed under $r_i$ and
$\neg r_i$. The hull $[\min(\mathrm{lb}_k^{r_i}, \mathrm{lb}_k^{\neg r_i}), \max(\mathrm{ub}_k^{r_i},
\mathrm{ub}_k^{\neg r_i})]$ is sound for $k$ without assuming either phase of ReLU $i$.
\end{theorem}

Proof is deferred to Appendix~\ref{sec:appendix-proofs}. With the implication graph, we can 
also use the SAT closure of the graph to derive new information. 

\subsection{SAT Closure of the Implication Graph}
\label{subsec:sat-closure}

The implication graph is built from binary implications and unit lemmas of
ReLU phase fixes, which are boolean variables. Thus, the implication graph
can be represented as a SAT formula $\Sigma$, where each unit lemma $r$ becomes the clause $(r)$, and
each edge $(r, r') \in E$ becomes the clause $(\neg r \vee r')$.

At each BaB node, let $R$ be the phase fixes already committed along the path from the
root to that node. Checking $\Sigma \cup R$ with a SAT solver is free relative to the tightening
oracle; if it is unsatisfiable, the node is infeasible and can be pruned without any further
splitting. Theorem~\ref{thm:sat-closure-pruning} formalizes the soundness of the closure.

\begin{theorem}[Soundness of SAT-Closure Pruning]
\label{thm:sat-closure-pruning}
Let $R$ be the phase fixes fixed at a BaB node and $\Sigma$ the implication graph's clauses. If
$\Sigma \cup R$ is unsatisfiable, the subproblem is UNSAT.
\end{theorem}

Proof is deferred to Appendix~\ref{sec:appendix-proofs}. $\Sigma \cup R$ only checks the clauses
the graph already has, computed against $Q$ alone at the root. Deeper in the tree, $R$ has already
tightened the problem substantially, and a pin that did not refute at the root may refute now, or force a phase it
did not force before. Catching that requires rerunning the lookahead procedure itself, not just
recombining existing clauses.

\subsection{Reprobing to Update the Implication Graph}
\label{subsec:reprobing}

The implication graph was built once, before any splits were made, so its facts only reflect what
is true at the root. By the time BaB reaches a deeper node, several ReLUs have already been fixed,
and those fixes narrow the space of remaining inputs. A phase fix that looked possible at the root
may now be provably impossible once the earlier fixes are accounted for, or it may now force some
other neuron's phase where it could not before. Reprobing catches these newly-provable facts by
simply rerunning the lookahead procedure at the current node instead of at the root.

Formally, let $Q$ denote the original verification query. Let
\[
Q_R := Q \wedge \bigwedge_{r \in R} r
\]
denote the query restricted to the current node's fixed phases. Reprobing is
Algorithm~\ref{alg:construct-implication-graph} rerun against $Q_R$ instead of $Q$, restricted to
the ReLUs still unstable at the node: each still-unstable neuron is branched on its active and
inactive phase, the tightening oracle is rerun under $Q_R$ together with that pin, and any
resulting infeasibility or forced phase is recorded as a unit lemma or edge, exactly as before. By
Theorem~\ref{thm:soundness-construction} with $Q$ replaced by $Q_R$, the new units and edges hold
in every model of $Q_R$. The trigger condition below ensures every fix in $R$ is entailed by $Q$,
so $Q_R$ has the same models as $Q$ and the new facts enter $\Sigma$ as root-global facts. This keeps the implication graph up to date with the present subproblem.

Reprobing is not run at every node. A pass sweeps every unstable neuron in both phases, so it
costs on the order of the initial construction, and running one per node would dominate the search. Instead a reprobe is triggered when two conditions hold. First, the search
must be at a point where no branching decisions are in force, so that the facts derived are
conditioned only on phases already established globally rather than on a path that is about to be
undone. Second, the set of established phases must have grown since the last pass, since $Q_R$ is
otherwise unchanged and a repeat sweep would rederive exactly what is already in $\Sigma$.

\subsection{Cut Vivification}
\label{subsec:cut-vivification}

With the implication graph, beyond using it to directly determine infeasibility, we can 
use it to vivify boolean cuts. Recall that a boolean cut is a collection of ReLU phase fixes
that are globally infeasible. The process of vivifying a cut is to remove phase fixes from 
the collection, creating a smaller set of phase fixes and a stronger cut. The implication 
graph provides the machinery to vivify cuts. 

Concretely, both BICCOS and PICID cuts are mined from subproblems determined infeasible, 
coming directly from the BaB search procedure. All of the phase fixes are used 
to derive the cut in some form, but in most cases, not all of them are necessary. 
The implication graph records which phase fixes are consequences of each other, so 
naturally, we can use the implication graph to determine which phase fixes in 
the cut are consequences of others.

\begin{algorithm}[t]
\caption{Cut Vivification}
\label{alg:vivify}
\begin{algorithmic}[1]
\REQUIRE cut $S$, clause formula $\Sigma$, query $Q$
\ENSURE vivified cut $S' \subseteq S$
\FOR{each $r \in S$}
    \IF{$\Sigma \cup (S \setminus \{r\})$ propagates $r$}
        \STATE $S \gets S \setminus \{r\}$
    \ENDIF
\ENDFOR
\STATE propagate $S$ through $\Sigma$
\IF{propagation derives both $r$ and $\neg r$}
    \RETURN the roots of the antecedent chains of $r$ and $\neg r$
\ENDIF
\IF{$\Sigma$ is unsatisfiable under the assumptions $S$}
    \RETURN the failed assumptions $S' \subseteq S$
\ENDIF
\STATE order $S$ as $r_{(1)}, \dots, r_{(n)}$, most tightening first
\STATE $P \gets \emptyset$
\FOR{$j = 1, \dots, n$}
    \STATE $P \gets P \cup \{r_{(j)}\}$
    \STATE $P \gets P \cup \{r' : \Sigma \cup P \text{ propagates } r'\}$
    \IF{propagation yields a conflict}
        \RETURN $\{r_{(1)}, \dots, r_{(j)}\}$
    \ENDIF
    \STATE tighten bounds under $P$ with a restricted oracle
    \IF{$P$ is infeasible}
        \RETURN $\{r_{(1)}, \dots, r_{(j)}\}$
    \ENDIF
\ENDFOR
\RETURN $S$
\end{algorithmic}
\end{algorithm}

Algorithm~\ref{alg:vivify} shows the cut vivification procedure. The first observation 
is that if a phase fix in a cut is implied in the implication graph by the
other phase fixes in the cut, then the phase fix is redundant and can be removed.
The following theorem formalizes this observation.

\begin{theorem}[Redundant Phase Fixes]
\label{thm:redundant-fix}
Let $S$ be a cut, let $r \in S$, and write $S' = S \setminus \{r\}$. If
\[
\Sigma \wedge \bigwedge_{r' \in S'} r' \models r,
\]
then $S'$ is itself a cut.
\end{theorem}

Proof is deferred to Appendix~\ref{sec:appendix-proofs}. The next observation is
that if a cut is unsatisfiable against the SAT closure of the implication graph,
then a subset of its phase fixes is already infeasible and that subset is a
stronger cut.

That subset is read directly off the implication graph. Assume every phase fix of the cut and
propagate through the graph. Each newly derived phase fix is forced by a single edge, so it can be
traced backwards along that edge to the fix that forced it, and that fix in turn to the one before
it, until the chain ends at either a phase fix of the cut or a unit lemma. A conflict occurs when
some phase fix and its complement are both derived. Tracing each of the two back to where its chain
began identifies exactly the phase fixes of the cut that participated in the refutation, and that
set is the unsatisfiable core. Since every derived fix has exactly one antecedent, each chain has a
single root, so a conflict found this way rests on at most two phase fixes of the cut. A cut of any
length therefore collapses in one step to a cut of size two, or to a single phase fix when one of
the two chains begins at a unit lemma.

In practice this trace is not performed by hand. The graph is already maintained as the formula
$\Sigma$, so the phase fixes of the cut are handed to the SAT
solver as assumptions and the solver is asked whether $\Sigma$ survives them. When it does not, the
solver reports which of the assumptions its refutation used, and that set of failed assumptions is
the core. Conflict analysis performs the same backward trace, over a resolution derivation rather
than over a single chain of edges, which is what allows the procedure to keep working once $\Sigma$
holds clauses that are not binary, such as previously vivified cuts. For those, propagation alone
may no longer find the conflict and the core may be larger than two, but the interface is unchanged.
The solve is conflict bounded, so its cost remains negligible against a single call to the
tightening oracle. Lemma~\ref{lem:graph-core} formalizes the core extraction.

\begin{lemma}[Core from the Implication Graph]
\label{lem:graph-core}
Suppose propagating $S$ through the binary fragment of $\Sigma$, its unit and binary
clauses, derives both $r$ and $\neg r$, and let $a_1$ and $a_2$ be
the roots of their antecedent chains. Then $S' = \{a_1, a_2\} \cap S$ satisfies $|S'| \le 2$ and
$\Sigma \cup S'$ is unsatisfiable.
\end{lemma}

Proof is deferred to Appendix~\ref{sec:appendix-proofs}. The intersection with $S$ covers the case
where a chain begins at a unit lemma rather than at the cut, as such a root holds under $\Sigma$ alone
and contributes nothing to $S'$. When both chains begin at unit lemmas, $S' = \emptyset$ and
$\Sigma$ is unsatisfiable on its own, which by Theorem~\ref{thm:sat-closure-pruning} with
$R = \emptyset$ means the query itself is UNSAT and the search can stop.
Corollary~\ref{cor:core-cut} formalizes this fact.

\begin{corollary}[Refutation by the Implication Graph]
\label{cor:core-cut}
Let $S$ be a cut and let $S' \subseteq S$ be an unsatisfiable core of $\Sigma$ under the assumptions
$S$. Then $S'$ is itself a cut.
\end{corollary}

This is Theorem~\ref{thm:sat-closure-pruning} applied with $R = S'$, where a set of phase fixes that
cannot be extended to a model of $\Sigma$ cannot be extended to a model of $Q$ either. Together with
Lemma~\ref{lem:graph-core} it licenses returning the core in place of the cut. Using exclusively
the implication graph has one drawback: it sees only binary implications and unit lemmas, so the
first two steps add negligible overhead but cannot act on combinations of phase fixes. For that we
turn to the tightening oracle, using a variant of the lookahead procedure.

Algorithm~\ref{alg:construct-implication-graph} pins one phase fix and asks whether the query
survives it, but nothing in the procedure requires the pin to be a single fix. Pinning several at
once and running the same restricted tightening oracle asks the same question of a set, and the
answer is what the criterion for vivification needs. The remaining phase fixes of the cut are
therefore added to a pin set one at a time, and after each addition the bounds are tightened. If the
pin set becomes infeasible after $j$ additions, those $j$ phase fixes are a cut on their own and
replace $S$. This is the only stage that consults the network rather than the graph, and it is the
only one that can act on combinations of phase fixes, because a combination is precisely what it
pins.

Since the fixes are added in order, the procedure can only return a prefix of the order it was
given. We sort the fixes by the tightening each one produced
when it was probed during the construction of the implication graph, most tightening first, breaking
ties by the instability of the neuron under the refined bounds. Both quantities are by-products of
Algorithm~\ref{alg:construct-implication-graph} and are already stored, so the ordering costs
nothing, and placing the most constraining fixes first makes an infeasibility surface at the
smallest $j$.

The two boolean stages continue to pay off inside this one. After each fix is pinned, the phase
fixes it implies in $\Sigma$ are added to the pin set as well. These additions are free and can only
help the oracle, and if propagation conflicts outright, the prefix is returned without invoking the
oracle at all. This is sound because each added fix is entailed by $\Sigma$ together with the
prefix, so the pin set and the prefix have the same models by the argument of
Theorem~\ref{thm:redundant-fix}. Theorem~\ref{thm:descent} formalizes the soundness of this.

\begin{theorem}[Soundness of the Descent]
\label{thm:descent}
Let $P_j = \{r_{(1)}, \dots, r_{(j)}\}$ and let $\hat{P}_j$ denote its closure under propagation in
$\Sigma$. If the tightening oracle finds $Q \wedge \bigwedge_{r \in \hat{P}_j} r$ infeasible, then
$P_j$ is a cut.
\end{theorem}

This is a consequence of the previous theorems, as an empty bound implies infeasibility
by Theorem~\ref{thm:soundness-construction}, and the propagated fixes may be discarded from
$\hat{P}_j$ by Theorem~\ref{thm:redundant-fix}. A cut that reaches the end of
the descent unchanged is certified minimal with respect to the
tightening oracle.

\subsection{Inprocessing Framework with Lookahead}
\label{subsec:framework}

Figure~\ref{fig:inprocessing-loop} shows how the pieces fit together over the course of a solve.
Before the search begins, the lookahead procedure 
probes every unstable ReLU in both phases, and its
results seed three channels at once. Unit lemmas and binary implications become the clauses of
$\Sigma$, and the hull of each pair of probes tightens the root bounds.

The search then proceeds as usual, with two additions at each subproblem. On entering a subproblem
with committed phase fixes $R$, the solver first checks $\Sigma \cup R$. If it is unsatisfiable the
subproblem is infeasible by Theorem~\ref{thm:sat-closure-pruning} and is discarded before any bound
computation. Otherwise the phase fixes that $\Sigma$ propagates from $R$ are clamped into the
subproblem's bounds. Both are propagation over the clause database, so their cost is negligible.

When a subproblem is proven UNSAT, the underlying cut mechanism mines a cut from it as before, but
the cut is now vivified by Algorithm~\ref{alg:vivify} before it is installed. The shortened cut goes
into the cut pool and into $\Sigma$, where it strengthens every subsequent check, and any cut that
vivification reduces to a single phase fix is a unit lemma and enters the implication graph as a
node fact. Cuts are added to $\Sigma$ only after their own vivification attempt, since a cut already
present in $\Sigma$ refutes its own assumptions and the resulting core is the cut itself.

Those unit lemmas grow the set of established phases, which is the condition under which reprobing
runs. When the search next reaches a point with no branching decisions in force, the lookahead
procedure is rerun against $Q_R$ over the still-unstable neurons, and whatever it derives returns to
$\Sigma$ as root-global facts. The result is a loop rather than a pipeline: the graph informs the search,
the search produces cuts, vivification turns those cuts back into graph facts, and those facts
trigger the probing that grows the graph further.

\begin{figure}[t]
    \centering
    \resizebox{\linewidth}{!}{\begin{tikzpicture}[
    box/.style={rectangle, rounded corners=2pt, draw, thick, align=center,
                minimum height=11mm, minimum width=21mm, font=\large, fill=blue!18},
    probe/.style={box, fill=orange!22},
    cut/.style={box, fill=black!8},
    lbl/.style={font=\normalsize, align=center},
    ->/.style={-{Stealth[length=2.4mm]}}
]

\node[probe] (P) at (0, 2.6)     {Lookahead\\Probe};
\node[box]   (G) at (0, 0)       {Implication\\Graph $\Sigma$};
\node[box]   (B) at (4.3, 0)     {BaB\\Subproblem};
\node[cut]   (C) at (4.3, -2.6)  {Mined\\Cut};
\node[cut]   (V) at (0, -2.6)    {Vivification};
\node[probe] (R) at (-3.9, -1.3) {Reprobe};

\draw[->] (P) -- (G) node[lbl, midway, right] {units, edges,\\hull};

\draw[->] (G) -- (B) node[lbl, midway, above] {prune /\\clamp};
\draw[->] (B) -- (C) node[lbl, midway, right] {UNSAT};
\draw[->] (C) -- (V);

\draw[->] (V) -- (G) node[lbl, midway, right] {shortened\\cut};
\draw[->] (V) -- (R) node[lbl, midway, below left=-0.5mm] {fixed set\\grew};
\draw[->] (R) -- (G) node[lbl, midway, above left=-0.5mm] {facts\\under $Q_R$};

\end{tikzpicture}}
    \caption{Inprocessing over the course of a solve. Lookahead probe constructs the implication
    graph; cuts mined from UNSAT subproblems are vivified back
    into the graph; new phase fixes trigger reprobing.}
    \label{fig:inprocessing-loop}
\end{figure}
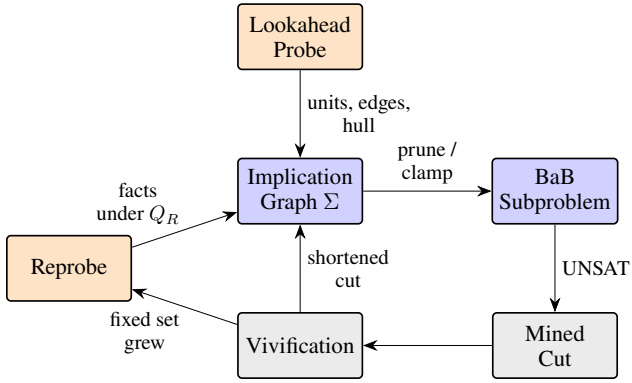

\section{Evaluation}
\label{sec:evaluation}

To evaluate the effectiveness of our proposed inprocessing framework, we implemented the framework 
in two state-of-the-art BaB-based verifiers in Marabou and \abcrown. These two solvers represent 
two different solver paradigms: Marabou~\citep{wu2024marabou2} is a CPU-based verifier that employs an SMT-based
search procedure and PICID, while \abcrown~\citep{wang2021beta} uses GPU-accelerated bound propagation and BICCOS. We evaluate our framework against
the most recent versions of these verifiers, which include both PICID and BICCOS respectively.

\subsection{Case Study on Marabou}
\label{subsec:marabou}

\subsubsection{Experimental Setup}
\label{subsec:setup-marabou}
For experimentation on Marabou, we evaluated our framework on several standard benchmarks
from recent literature using Marabou, including the ACASXu, MNIST and SafeNLP benchmarks.
The ACASXu benchmark~\citep{julian2016policy,katz2017reluplex} is a set of neural networks for airborne collision avoidance systems,
and has been widely used to evaluate different verifiers. The MNIST dataset~\citep{lecun1998gradient}
includes various feed-forward neural networks for handwritten digit recognition.
SafeNLP is a set of neural networks for natural language processing tasks and has been a staple
benchmark in the most recent editions of the VNN-COMP~\citep{brix2023vnncomp,brix2024vnncomp,kaulen2025vnncomp}.

CaDiCaL~\citep{BiereFallerFazekasFleuryFroleyks-CAV24} was used as the SAT solver for the implication graph, and the tightening oracle
included one DeepPoly~\citep{singh2019deeppoly} pass and a simplex feasibility check capped at 400 pivots.
The experiments were run on a server with 2.6-GHz AMD CPUs, with 4 CPU cores allocated per benchmark.
Each benchmark was given a 20 minute time limit and a 32 GB memory limit.

\subsubsection{Results}
\label{subsec:results-marabou}

Table~\ref{tab:marabou} shows the results of the inprocessing framework on Marabou. 
In total, the framework leads to fewer timeouts on every
benchmark. It solves substantially more UNSAT instances on every benchmark, and on SAT
instances it matches PICID on ACASXu, solves two more on MNIST, and solves fewer only
on SafeNLP. The SafeNLP regression reflects an asymmetry
between deducing UNSAT versus SAT. Proving unsatisfiability is a universal claim, 
requiring every subdomain to be proven UNSAT. Proving satisfiability requires only 
one counterexample, which is why attack algorithms as shown in Section~\ref{subsec:abcrown} 
are usually more effective at deducing SAT than BaB. Inprocessing slightly changes the
order in which domains are explored, which by chance can leave a domain with a
counterexample explored later. Cactus plots for both runtime
and visited states on the Marabou benchmarks are presented in Appendix~\ref{sec:appendix-cactus}.

\begin{table}[t]
\centering
\footnotesize
\setlength{\tabcolsep}{1.4pt}
\begin{tabular*}{\linewidth}{@{\extracolsep{\fill}}lrrrrrrr}
\toprule
 & \multicolumn{3}{c}{Instances} & \multicolumn{2}{c}{Avg.\ time (s)} & \multicolumn{2}{c}{Avg.\ states} \\
\cmidrule(lr){2-4}\cmidrule(lr){5-6}\cmidrule(lr){7-8}
Configuration & UNSAT & SAT & T/O & UNSAT & SAT & UNSAT & SAT \\
\midrule
\multicolumn{8}{l}{\textit{ACASXu}} \\
PICID           & 99  & 24 & 40 & 142.2 & \textbf{53.2}  & 1724.1 & \textbf{502.9} \\
+ inprocessing & \textbf{124} & 24 & \textbf{15} & \textbf{75.6} & 115.5 & \textbf{795.1} & 993.5 \\
\addlinespace
\multicolumn{8}{l}{\textit{MNIST}} \\
PICID           & 102 & 27 & 122 & \textbf{12.8} & \textbf{172.0} & 16.0 & 3041.1 \\
+ inprocessing & \textbf{137} & \textbf{29} & \textbf{85} & 37.7 & 185.2 & \textbf{13.3} & \textbf{2685.2} \\
\addlinespace
\multicolumn{8}{l}{\textit{SafeNLP}} \\
PICID           & 242 & \textbf{596} & 238 & 42.5 & \textbf{1.2} & 3982.7 & \textbf{84.5} \\
+ inprocessing & \textbf{266} & 575 & \textbf{235} & \textbf{25.9} & 1.7 & \textbf{2308.8} & 134.2 \\
\bottomrule
\end{tabular*}
\caption{Results on Marabou.}
\label{tab:marabou}
\end{table}

Table~\ref{tab:marabou-overhead} shows the overhead of the lookahead probing and reprobing phases
on Marabou. On SafeNLP probing takes up less than 0.1\% of the total wall clock time, and on
ACASXu 2.6\%. On MNIST it accounts for 19.2\%, though the gap between the median and mean
per-instance probing times shows this is driven by a few outliers, and inprocessing still solves 35
more UNSAT instances there.
Appendix~\ref{sec:appendix-yield} reports how often each component fires and succeeds.

\begin{table}[t]
\centering
\footnotesize
\begin{tabular*}{\linewidth}{@{\extracolsep{\fill}}lrrrr}
\toprule
 & \multicolumn{3}{c}{Probing time (s)} & \\
\cmidrule(lr){2-4}
Benchmark & Median & Mean & Max & Wall clock (\%) \\
\midrule
ACASXu  & 3.83  & 5.43   & 13.85  & 2.6 \\
MNIST   & 38.42 & 135.31 & 491.73 & 19.2 \\
SafeNLP & 0.24  & 0.24   & 0.41   & ${<}0.1$ \\
\bottomrule
\end{tabular*}
\caption{Overhead of inprocessing on Marabou.}
\label{tab:marabou-overhead}
\end{table}

\subsection{Case Study on \abcrown}
\label{subsec:abcrown}

\subsubsection{Experimental Setup}
\label{subsec:setup-abcrown}
For experimentation on \abcrown, we evaluated our framework on the CIFAR and TinyImageNet
benchmarks drawn from recent literature using \abcrown~\citep{zhou2024scalable,zhou2025clip}. These consist of CNNs
for image classification ranging from 14.4 million to 31.6 million parameters. The
CIFAR benchmarks include adversarially trained networks, making them challenging verification
tasks.

CaDiCaL~\citep{BiereFallerFazekasFleuryFroleyks-CAV24} was used as the
SAT solver for the implication graph through the PySAT Python API, 
with the tightening oracle implemented as a single
$\alpha$-CROWN bound propagation pass.
The experiments were run on a server with 2.6-GHz AMD CPUs and NVIDIA A100 GPUs, with
96 CPU cores and one A100 GPU allocated per benchmark. Each benchmark was given a
128 GB memory limit.

\subsubsection{Results}
\label{subsec:results-abcrown}

Table~\ref{tab:abcrown} presents the results of the inprocessing framework. \abcrown 
first runs PGD attack to find counterexamples, so on all four benchmarks all the 
SAT instances are solved without entering BaB. On UNSAT instances, the
inprocessing framework solves more instances and reduces the average
time on all four benchmarks. Cactus plots for both runtime and visited states
are presented in Appendix~\ref{sec:appendix-cactus}.

\begin{table}[t]
\centering
\footnotesize
\setlength{\tabcolsep}{4pt}
\begin{tabular*}{\linewidth}{@{\extracolsep{\fill}}lrrrrrr}
\toprule
 & \multicolumn{3}{c}{Instances} & \multicolumn{2}{c}{Avg.\ time (s)} & Avg. \\
\cmidrule(lr){2-4}\cmidrule(lr){5-6}
Configuration & UNSAT & SAT & T/O & UNSAT & SAT & states \\
\midrule
\multicolumn{7}{l}{\textit{CIFAR CNN-A-Adv}} \\
BICCOS         & 91 & 100 & 9 & 5.1 & 0.1 & 2446.6 \\
+ inprocessing & \textbf{92} & 100 & \textbf{8} & \textbf{4.3} & 0.1 & \textbf{1583.3} \\
\addlinespace
\multicolumn{7}{l}{\textit{CIFAR CNN-A-Mix}} \\
BICCOS         & 88 & 94 & 18 & 13.3 & 0.0 & 4434.0 \\
+ inprocessing & \textbf{90} & 94 & \textbf{16} & \textbf{10.1} & 0.0 & \textbf{2712.8} \\
\addlinespace
\multicolumn{7}{l}{\textit{CIFAR CNN-B-Adv}} \\
BICCOS         & 95 & 70 & 35 & 17.9 & 0.1 & 9095.9 \\
+ inprocessing & \textbf{97} & 70 & \textbf{33} & \textbf{16.1} & \textbf{0.0} & \textbf{2652.3} \\
\addlinespace
\multicolumn{7}{l}{\textit{TinyImageNet}} \\
BICCOS         & 131 & 43 & 26 & 31.0 & 2.4 & \textbf{1683.5} \\
+ inprocessing & \textbf{135} & 43 & \textbf{22} & \textbf{30.0} & \textbf{1.0} & 1762.7 \\
\bottomrule
\end{tabular*}
\caption{Results on \abcrown. Time limits: CIFAR CNN-A-Adv and CIFAR CNN-A-Mix: 200s,
CIFAR CNN-B-Adv: 350s, TinyImageNet: 360s.}
\label{tab:abcrown}
\end{table}

Table~\ref{tab:abcrown-overhead} shows the overhead of the probing of the inprocessing framework. 
Overall, lookahead takes between 0.2\% and 6\% of the total wall clock time. 
Appendix~\ref{sec:appendix-yield} reports how often each component fires and succeeds.

\begin{table}[t]
\centering
\footnotesize
\setlength{\tabcolsep}{4pt}
\begin{tabular*}{\linewidth}{@{\extracolsep{\fill}}lrrrr}
\toprule
 & \multicolumn{3}{c}{Probing time (s)} & \\
\cmidrule(lr){2-4}
Benchmark & Median & Mean & Max & Wall clock (\%) \\
\midrule
CIFAR CNN-A-Adv & 0.18 & 0.18 & 0.30 & 0.2 \\
CIFAR CNN-A-Mix & 0.29 & 0.29 & 0.44 & 0.4 \\
CIFAR CNN-B-Adv & 4.12 & 4.14 & 7.33 & 2.0 \\
TinyImageNet    & 2.39 & 2.86 & 5.81 & 6.0 \\
\bottomrule
\end{tabular*}
\caption{Overhead of inprocessing on \abcrown.}
\label{tab:abcrown-overhead}
\end{table}

\section{Conclusion}
\label{sec:conclusion}

In this paper, we introduced an inprocessing framework for BaB-based neural network verification. 
The framework leverages the lookahead procedure to construct an implication graph of ReLU phases, 
which is then used to prune subproblems and vivify boolean cuts. We implemented the framework 
in two state-of-the-art BaB-based verifiers, Marabou and \abcrown, and evaluated it on several
benchmarks. Our results show that the inprocessing framework improves the performance of both verifiers.

This work suggests several directions for future work. The implication graph is currently exploited
through its SAT closure, which captures only the lemmas entailed over binary implications.
Propagation redundant (PR) clauses only need to leave satisfiability intact~\citep{heule2017short}, and
harvesting them from the graph, as SDCL does~\citep{heule2019encoding}, would
derive lemmas no closure can produce. A second direction is to lift inprocessing above the
propositional abstraction, reasoning about linear constraints rather than boolean phase variables to
prune subproblems and vivify cuts the boolean abstraction cannot. Finally, our lemmas are facts about
a network under an input domain, not about the query that produced them.
\citet{elsaleh2026incremental} show such conflicts survive query refinement, so carrying the graph
across queries would improve incremental verification.


\bibliography{bibli}


\clearpage

\appendix

\section{Proofs}
\label{sec:appendix-proofs}

\begin{proof}[Proof of Theorem~\ref{thm:soundness-construction}]
Suppose the oracle computes bound $[\mathrm{lb}_k, \mathrm{ub}_k]$ for neuron $k$ under phase fixes
$R$, and let $r''$ be a further fix inconsistent with $[\mathrm{lb}_k, \mathrm{ub}_k]$. Any point
satisfying $R' = R \cup \{r''\}$ satisfies $R$, so its pre-activation $k$ lies in
$[\mathrm{lb}_k, \mathrm{ub}_k]$, contradicting $r''$; thus $R'$ is infeasible.

\emph{Unit lemmas.} The algorithm records $\neg r$ as a unit lemma exactly when the oracle finds $r$
itself infeasible, i.e., the bounds computed under $\{r\}$ are already empty. By the argument above,
$\{r\}$ is globally infeasible.

\emph{Binary implications.} The algorithm adds the edge $(r, r_j')$ exactly when the bounds computed
under $\{r\}$ already confine ReLU $j$ to phase $r_j'$, i.e., every point satisfying $Q \wedge r$
satisfies $r_j'$. The complementary fix $\neg r_j'$ requires $j$ to be in the other phase, which no
such point can satisfy, thus $Q \wedge r \wedge \neg r_j'$ has no solution.
\end{proof}

\begin{proof}[Proof of Theorem~\ref{thm:hull-bound}]
$r_i$ and $\neg r_i$ are complementary, so every point of the input domain satisfies one of them,
and thus lies in $[\mathrm{lb}_k^{r_i}, \mathrm{ub}_k^{r_i}]$ or $[\mathrm{lb}_k^{\neg r_i},
\mathrm{ub}_k^{\neg r_i}]$ for neuron $k$. Their hull therefore contains every such point.
\end{proof}

\begin{proof}[Proof of Theorem~\ref{thm:sat-closure-pruning}]
Every clause of $\Sigma$ holds in every model of $Q$ by Theorem~\ref{thm:soundness-construction},
and every $r \in R$ holds at the node by construction, so a model of $Q \wedge \bigwedge_{r \in R} r$
would satisfy $\Sigma \cup R$. Since $\Sigma \cup R$ has no satisfying assignment, no such model
exists.
\end{proof}

\begin{proof}[Proof of Theorem~\ref{thm:redundant-fix}]
Since $S' \subseteq S$, every point satisfying $Q \wedge \bigwedge_{r' \in S} r'$ satisfies
$Q \wedge \bigwedge_{r' \in S'} r'$. Conversely, every clause of $\Sigma$ holds in every model of $Q$
by Theorem~\ref{thm:soundness-construction}, so a point satisfying $Q \wedge \bigwedge_{r' \in S'} r'$
satisfies $\Sigma \wedge \bigwedge_{r' \in S'} r'$ and satisfies $r$ as well.
Therefore, it satisfies $Q \wedge \bigwedge_{r' \in S} r'$. The two have the same models, so $S'$ is
infeasible whenever $S$ is.
\end{proof}

\begin{proof}[Proof of Lemma~\ref{lem:graph-core}]
$|S'| \le 2$ is immediate, as $S'$ is a subset of $\{a_1, a_2\}$. Consider any assignment satisfying
$\Sigma \cup S'$. Each root $a_i$ holds under it, either $a_i \in S'$ and it is assumed, or $a_i$ is
a unit lemma and its clause lies in $\Sigma$. The antecedent chain from $a_i$ is a sequence of phase
fixes in which each is forced by the previous one through an edge of the implication graph, and each
such edge contributes the clause $(\neg r \vee r')$ to $\Sigma$. Following the chain from $a_1$, the
assignment therefore satisfies $r$, and following the chain from $a_2$ it satisfies $\neg r$. No
assignment does both, so $\Sigma \cup S'$ is unsatisfiable.
\end{proof}

\section{Contribution of Each Component}
\label{sec:appendix-yield}

Table~\ref{tab:yield-marabou} shows how many times each component of the inprocessing framework 
attempted to fire, how many times it succeeded, and the average number of attempts and successes per 
benchmark. Table~\ref{tab:yield-abcrown} shows the same for \abcrown. Vivification fired nearly
every time on Marabou, and some of the time on \abcrown. SAT closure pruned only a small percentage 
of the subproblems attempted, but this is expected as SAT closure is attempted every search state.
For hull bounds, an attempt is an unstable neuron probed and a success is a bound tightening it
produced, so the ratio exceeds one where a single neuron tightens several bounds. Overall, 
on Marabou with the exception of SafeNLP the hull bound tightens each unstable neuron multiple times. 
On \abcrown, the hull bound tightens between 25\% and 35\% of the unstable neurons. 

\begin{table}[t]
\centering
\footnotesize
\setlength{\tabcolsep}{3.5pt}
\begin{tabular*}{\linewidth}{@{\extracolsep{\fill}}lrrrr}
\toprule
Component & Active & Avg.\ att. & Avg.\ succ. & Succ./att. \\
\midrule
\multicolumn{5}{l}{\textit{ACASXu}} \\
Vivification & 27  & 18.5  & 17.7  & 0.96 \\
SAT closure  & 147 & 728.7 & 8.0   & 0.01 \\
Hull bounds  & 147 & 182.3 & 733.2 & 4.02 \\
\addlinespace
\multicolumn{5}{l}{\textit{MNIST}} \\
Vivification & 16  & 15.5  & 14.4  & 0.93 \\
SAT closure  & 160 & 854.0 & 6.7   & 0.01 \\
Hull bounds  & 160 & 352.7 & 1021.9 & 2.90 \\
\addlinespace
\multicolumn{5}{l}{\textit{SafeNLP}} \\
Vivification & 61  & 27.3  & 26.3  & 0.96 \\
SAT closure  & 365 & 210.8 & 1.5   & 0.01 \\
Hull bounds  & 365 & 104.6 & 3.0   & 0.03 \\
\bottomrule
\end{tabular*}
\caption{Per-component yield on Marabou.}
\label{tab:yield-marabou}
\end{table}

\begin{table}[t]
\centering
\footnotesize
\setlength{\tabcolsep}{3.5pt}
\begin{tabular*}{\linewidth}{@{\extracolsep{\fill}}lrrrr}
\toprule
Component & Active & Avg.\ att. & Avg.\ succ. & Succ./att. \\
\midrule
\multicolumn{5}{l}{\textit{CIFAR CNN-A-Adv}} \\
Vivification    & 20  & 364.1  & 81.1  & 0.22 \\
SAT closure     & 27  & 796.3  & 46.6  & 0.06 \\
Hull bounds     & 27  & 862.6  & 214.1 & 0.25 \\
\addlinespace
\multicolumn{5}{l}{\textit{CIFAR CNN-A-Mix}} \\
Vivification    & 43  & 293.2  & 85.9  & 0.29 \\
SAT closure     & 55  & 1672.5 & 63.5  & 0.04 \\
Hull bounds     & 55  & 1238.5 & 320.6 & 0.26 \\
\addlinespace
\multicolumn{5}{l}{\textit{CIFAR CNN-B-Adv}} \\
Vivification    & 74  & 324.6  & 50.5  & 0.16 \\
SAT closure     & 84  & 5278.9 & 88.5  & 0.02 \\
Hull bounds     & 84  & 2286.1 & 721.3 & 0.32 \\
\addlinespace
\multicolumn{5}{l}{\textit{TinyImageNet}} \\
Vivification    & 84  & 271.2  & 128.1 & 0.47 \\
SAT closure     & 117 & 1039.7 & 69.4  & 0.07 \\
Hull bounds     & 117 & 1700.2 & 589.6 & 0.35 \\
\bottomrule
\end{tabular*}
\caption{Per-component yield on \abcrown.}
\label{tab:yield-abcrown}
\end{table}

\section{Additional Cactus Plots}
\label{sec:appendix-cactus}

Figures~\ref{fig:cactus-marabou-unsat-time} through~\ref{fig:cactus-marabou-sat-visited} report
runtime and domains visited on the Marabou benchmarks, and
figures~\ref{fig:cactus-abcrown-unsat-time} and~\ref{fig:cactus-abcrown-visited} report the same
for \abcrown. For \abcrown, only UNSAT instances are shown, as its PGD attack solves every SAT
instance without entering BaB. Legends report the number of instances solved by each
configuration.

\begin{figure}[t]
\centering
\includegraphics[width=\linewidth]{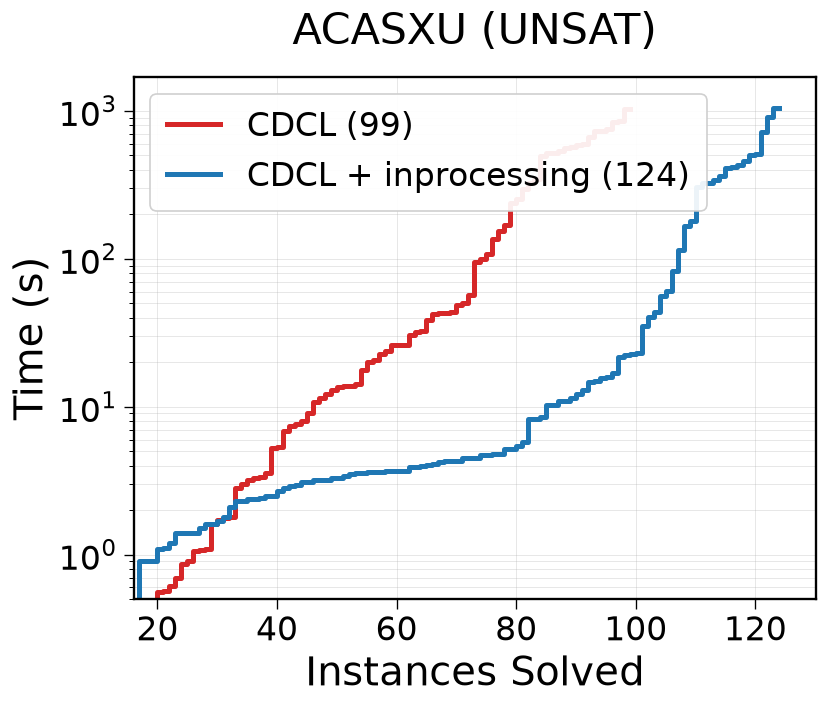}

\includegraphics[width=\linewidth]{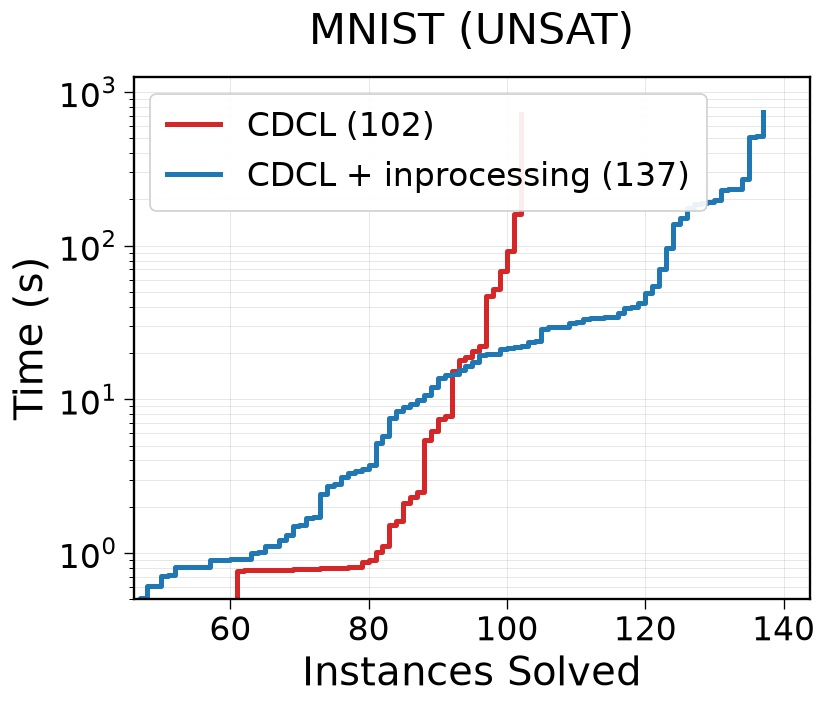}

\includegraphics[width=\linewidth]{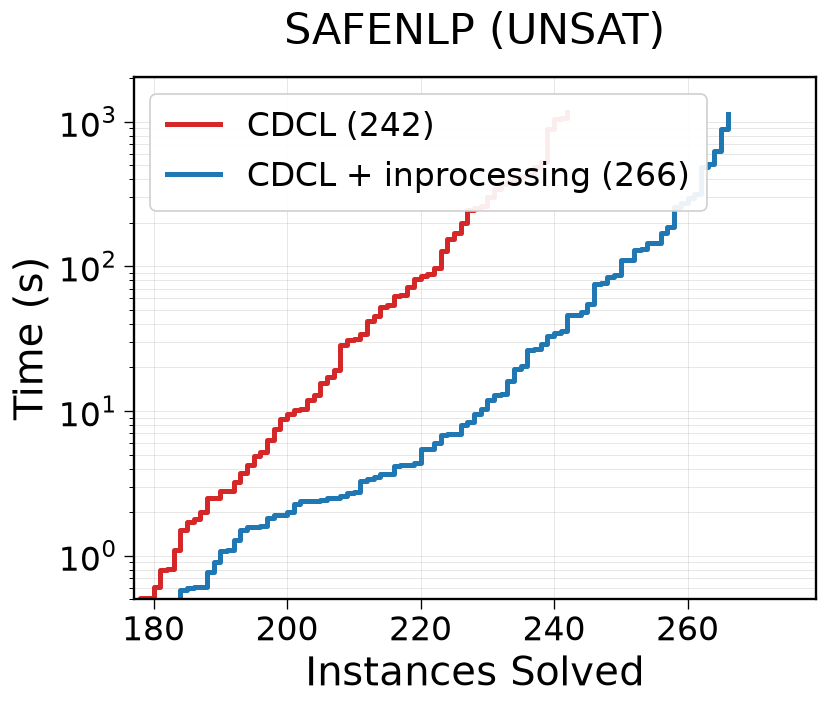}
\caption{Cactus plots of runtime on the UNSAT instances of the Marabou benchmarks.}
\label{fig:cactus-marabou-unsat-time}
\end{figure}

\begin{figure}[t]
\centering
\includegraphics[width=\linewidth]{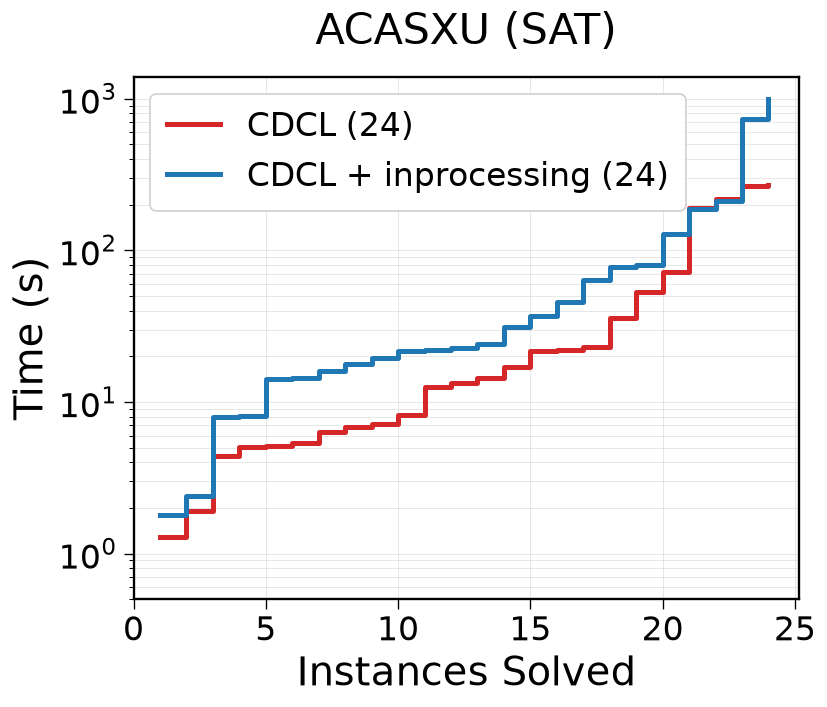}

\includegraphics[width=\linewidth]{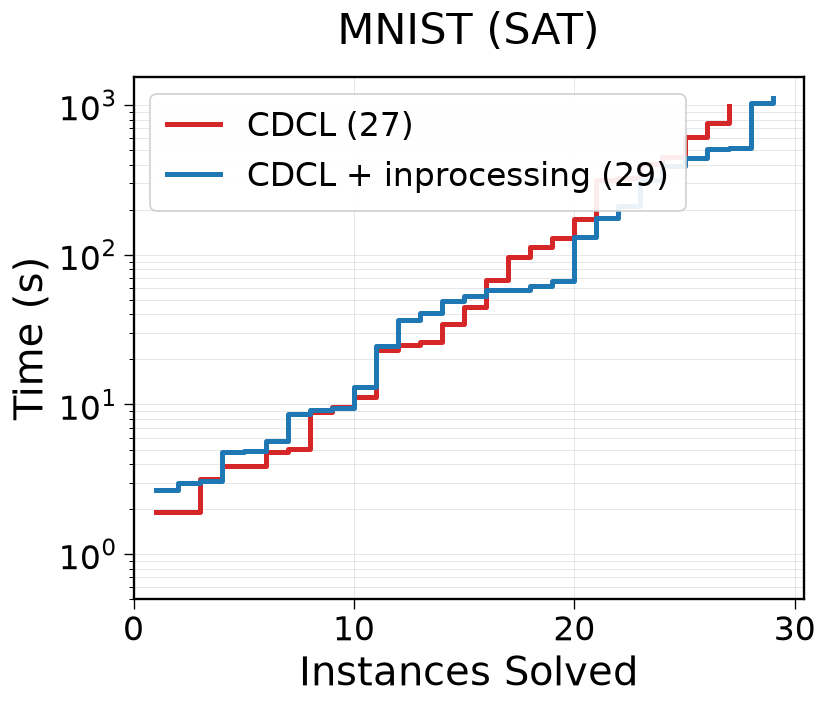}

\includegraphics[width=\linewidth]{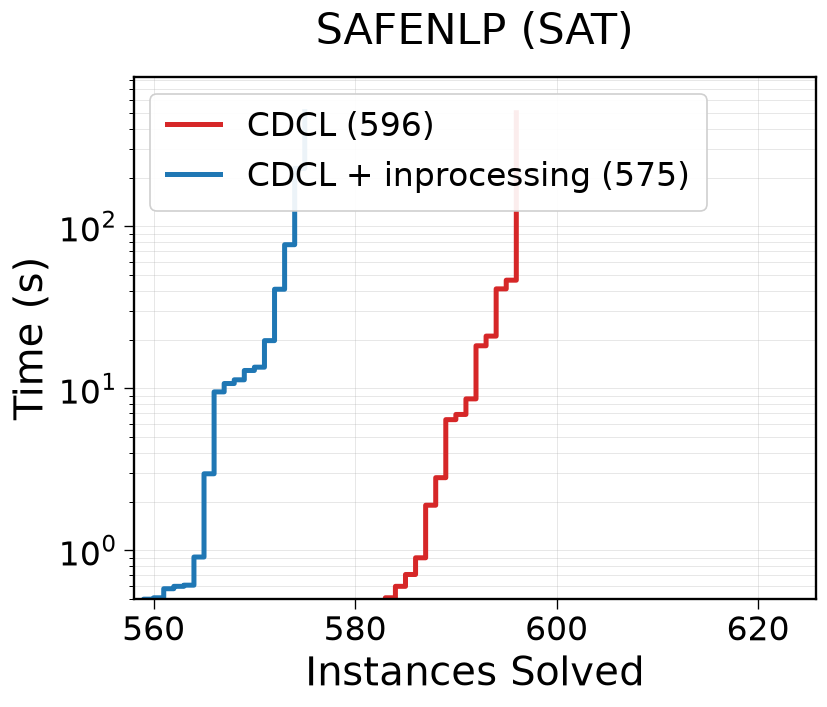}
\caption{Cactus plots of runtime on the SAT instances of the Marabou benchmarks.}
\label{fig:cactus-marabou-sat-time}
\end{figure}

\begin{figure}[t]
\centering
\includegraphics[width=\linewidth]{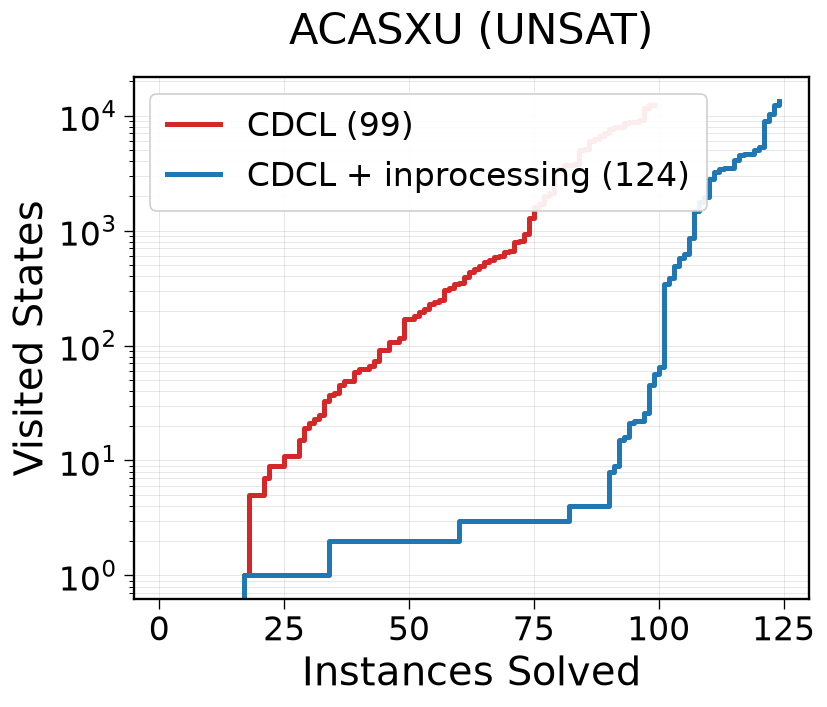}

\includegraphics[width=\linewidth]{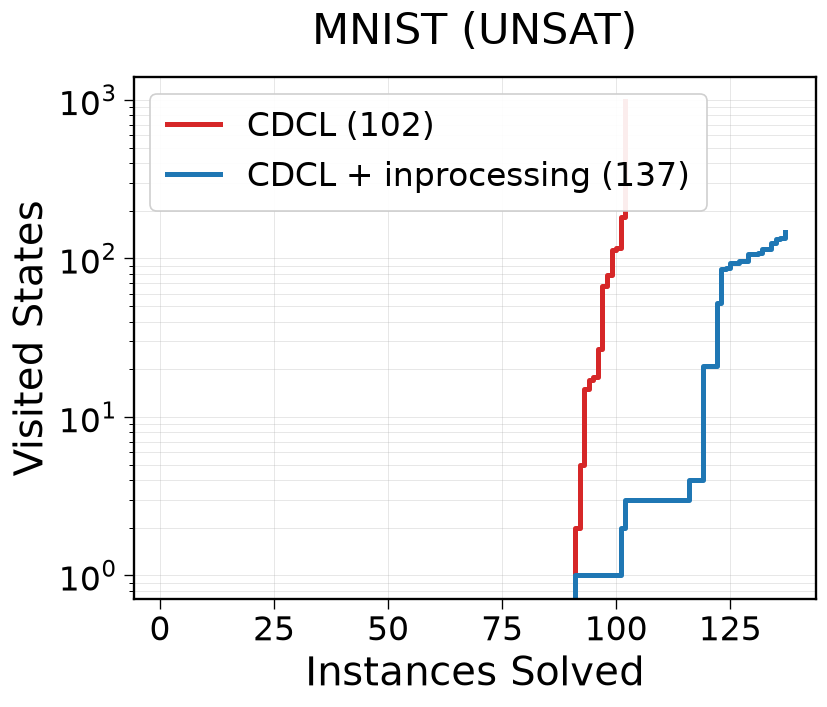}

\includegraphics[width=\linewidth]{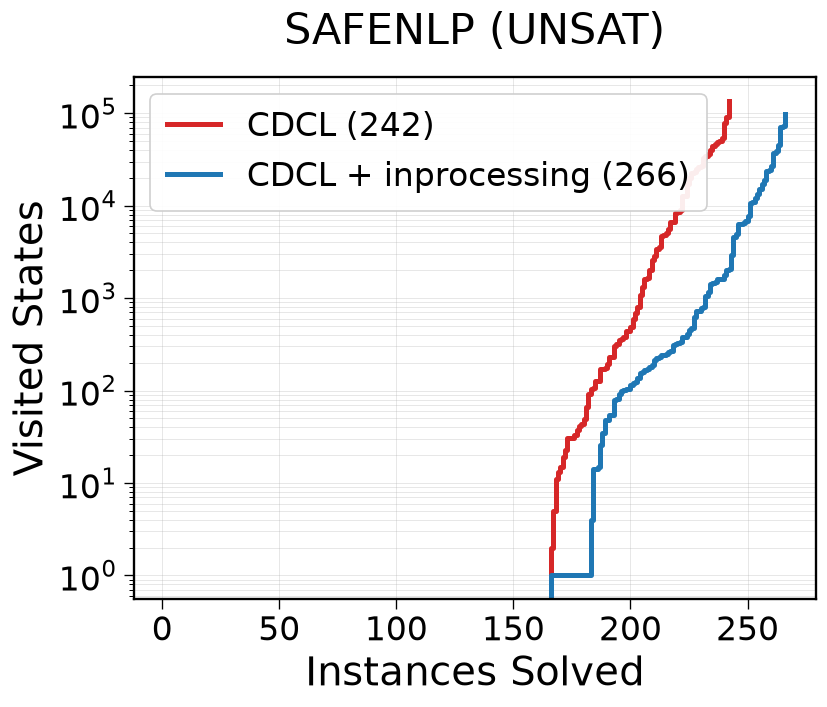}
\caption{Cactus plots of domains visited on the UNSAT instances of the Marabou benchmarks.}
\label{fig:cactus-marabou-unsat-visited}
\end{figure}

\begin{figure}[t]
\centering
\includegraphics[width=\linewidth]{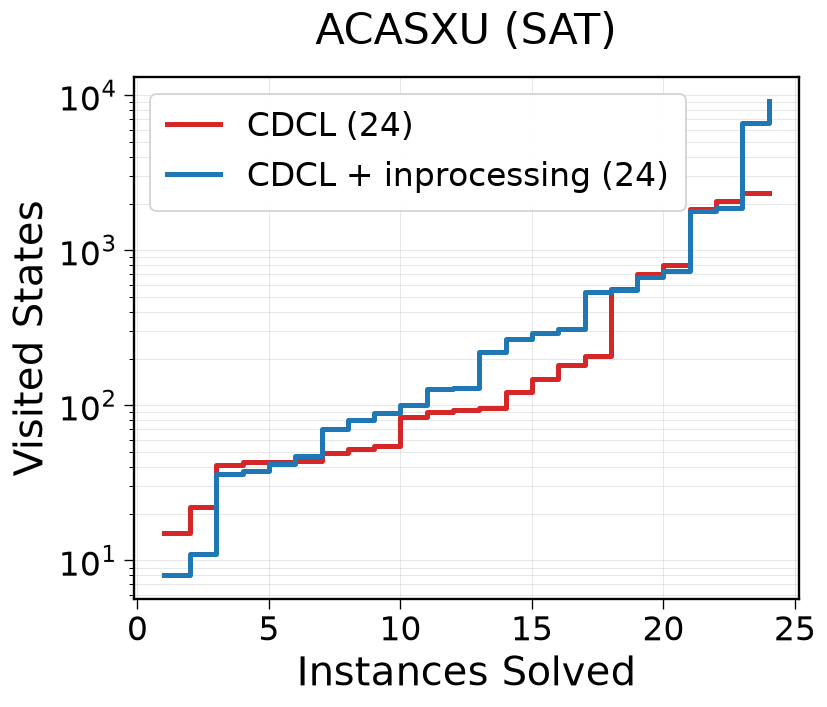}

\includegraphics[width=\linewidth]{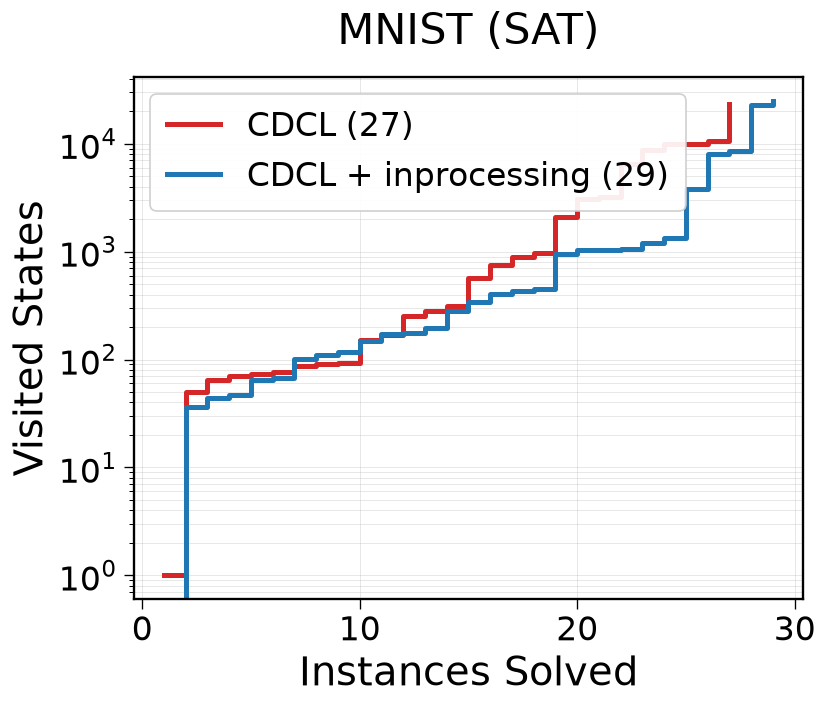}

\includegraphics[width=\linewidth]{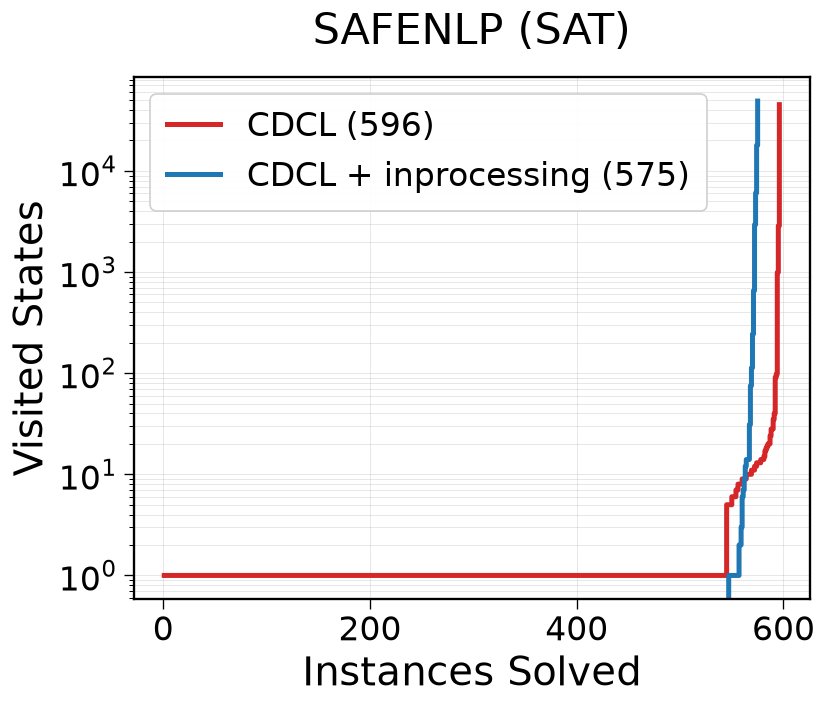}
\caption{Cactus plots of domains visited on the SAT instances of the Marabou benchmarks.}
\label{fig:cactus-marabou-sat-visited}
\end{figure}

\begin{figure}[t]
\centering
\includegraphics[width=0.75\linewidth]{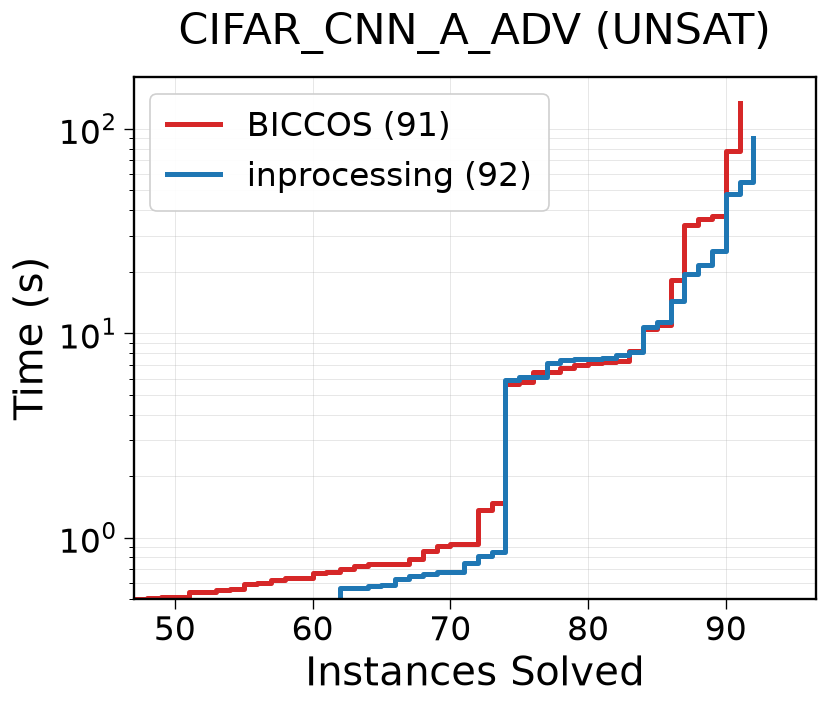}

\includegraphics[width=0.75\linewidth]{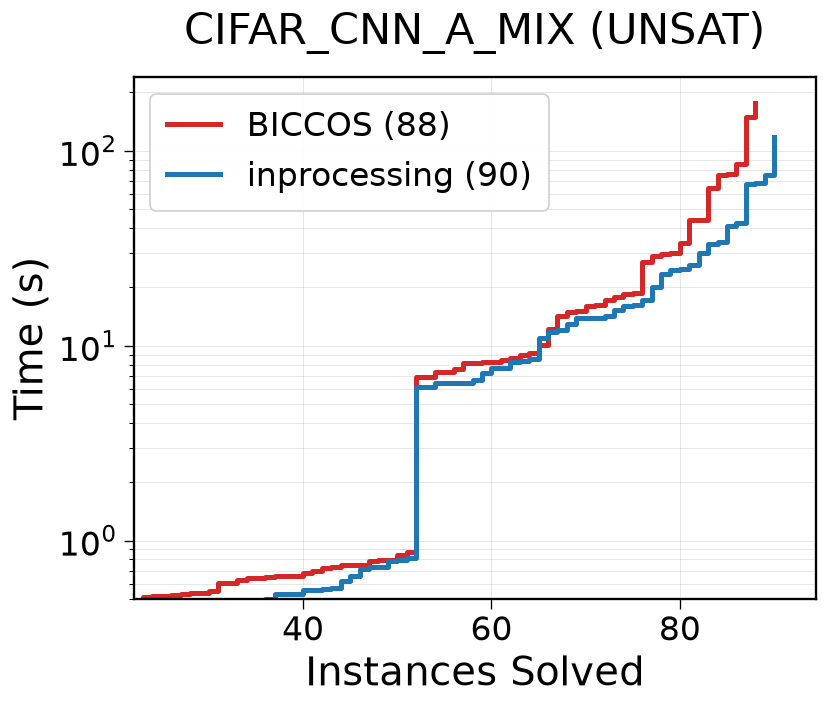}

\includegraphics[width=0.75\linewidth]{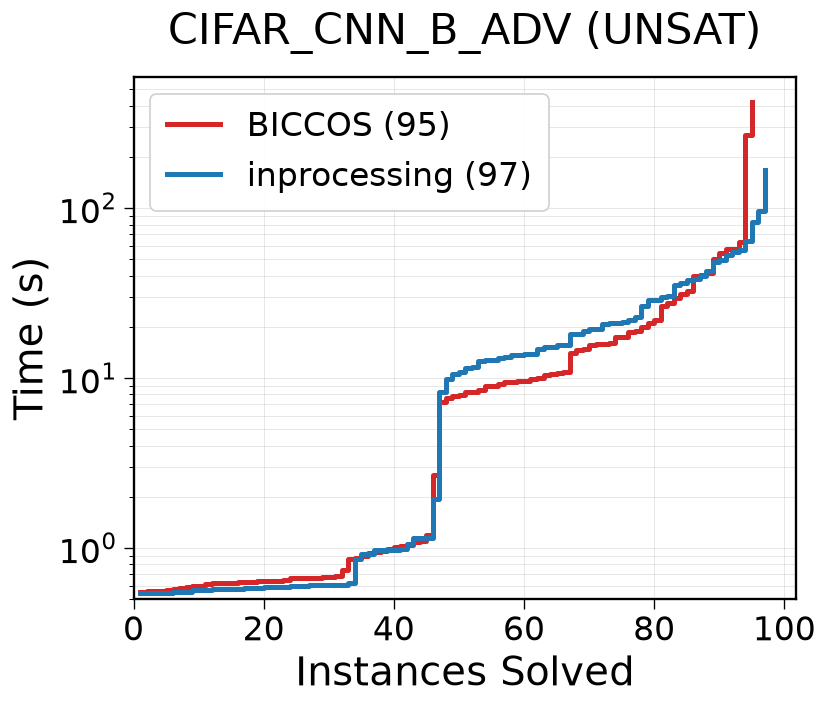}

\includegraphics[width=0.75\linewidth]{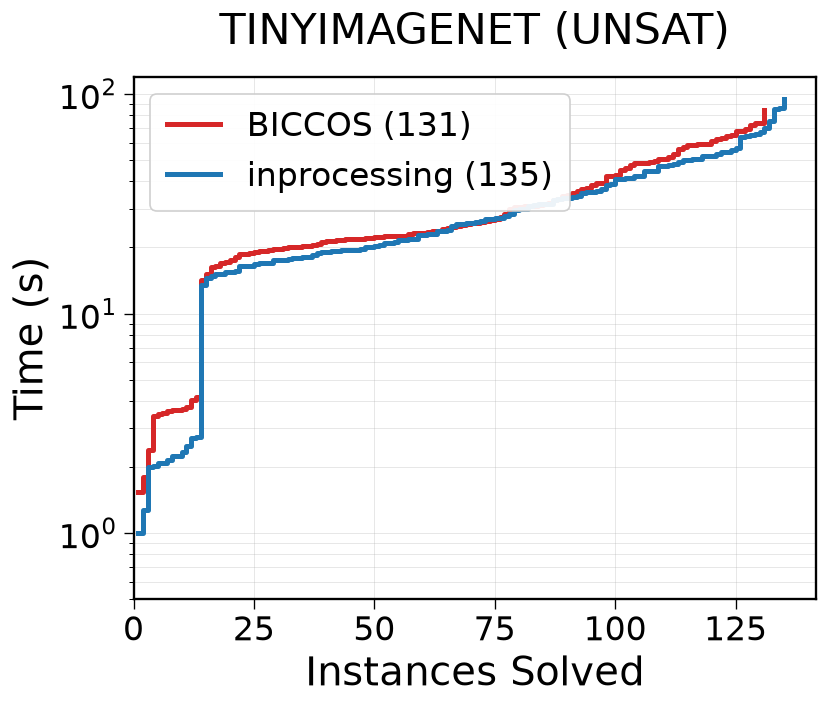}
\caption{Cactus plots of runtime on the UNSAT instances of the \abcrown benchmarks.}
\label{fig:cactus-abcrown-unsat-time}
\end{figure}

\begin{figure}[t]
\centering
\includegraphics[width=0.75\linewidth]{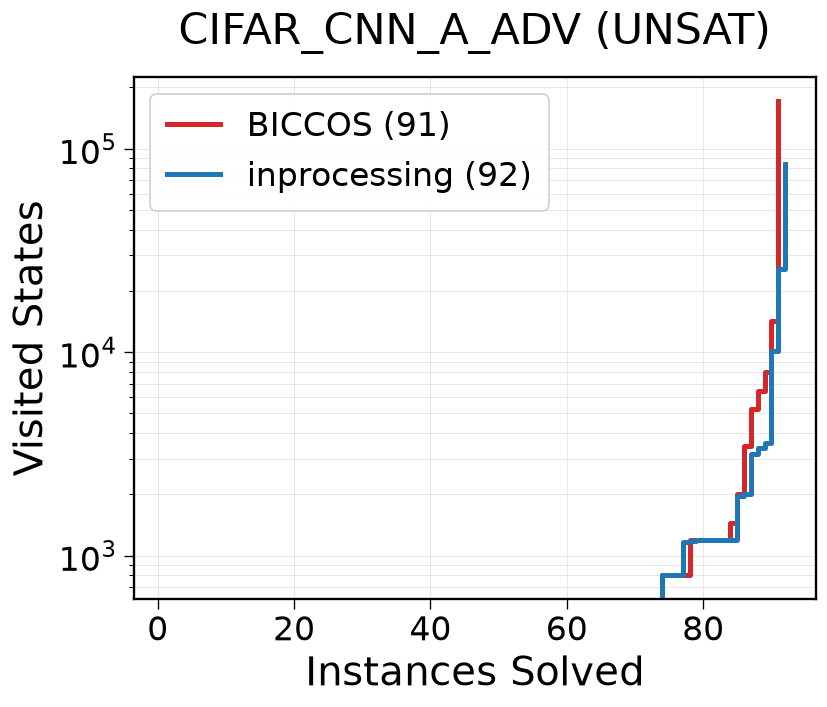}

\includegraphics[width=0.75\linewidth]{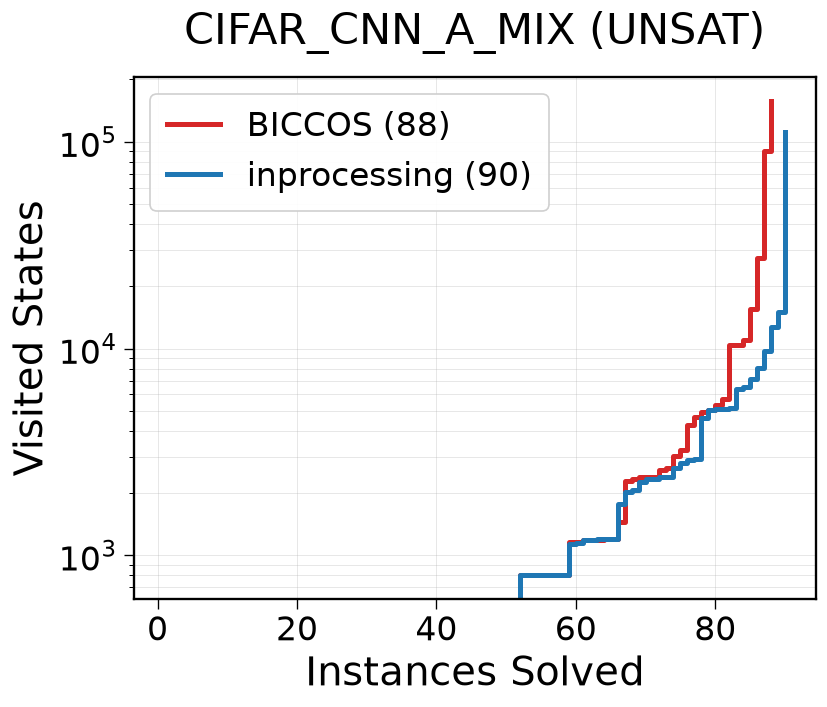}

\includegraphics[width=0.75\linewidth]{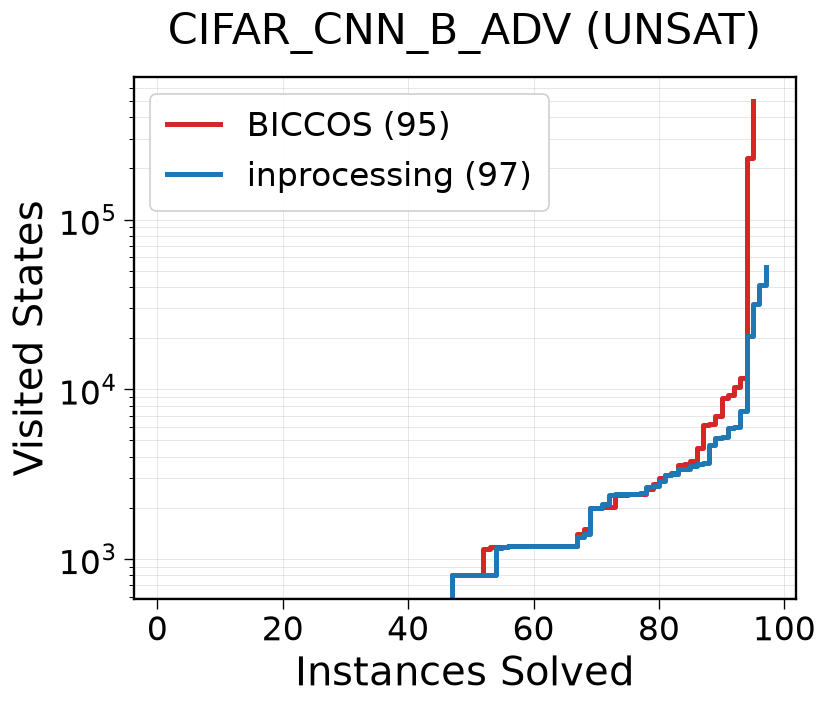}

\includegraphics[width=0.75\linewidth]{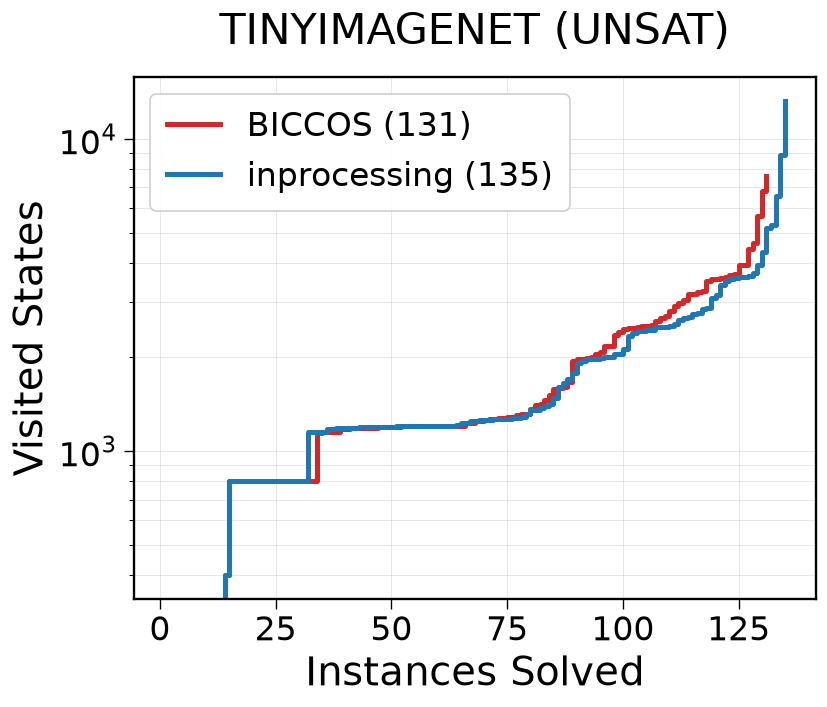}
\caption{Cactus plots of domains visited on the UNSAT instances of the \abcrown benchmarks.}
\label{fig:cactus-abcrown-visited}
\end{figure}

\end{document}